\documentclass{article}

\pdfoutput=1

\usepackage[final,nonatbib]{nips_2017}
\IfFileExists{\string~/Git/TeX/AVTcommands.sty}{\usepackage{\string~/Git/TeX/AVTcommands}}{\usepackage{AVTcommands}}

\usepackage[backend=bibtex, maxcitenames=2, useprefix, style=trad-abbrv, citestyle=numeric-comp, backref=true]{biblatex}
\DeclareFieldFormat{titlecase}{#1}
\usepackage{xpatch}
\xpatchbibmacro{pageref}{\printtext[parens]}{\addperiod\space\printtext}{}{}

\bibliography{\string~/Git/TeX/AVTcitations}

\usepackage[all]{xy}
\usepackage{microtype}

\usepackage{amsthm}
\newtheorem{theorem}{Theorem}
\newtheorem{definition}[theorem]{Definition}
\newtheorem{algorithm}[theorem]{Algorithm}
\newtheorem{result}[theorem]{Result}

\renewcommand{\[}{\begin{equation}}
\renewcommand{\]}{\end{equation}}
\def\<#1\>{\begin{align}#1\end{align}}

\begin{document}

\title{Techniques for proving Asynchronous Convergence results for Markov Chain Monte Carlo methods}

\author{
\textbf{Alexander Terenin}\thanks{Website: \textsl{\href{http://avt.im/}{http://avt.im/}}} \\
Petuum, Inc. and \\
Imperial College London \\
\textsc{\href{mailto:a.terenin17@imperial.ac.uk}{a.terenin17@imperial.ac.uk}}
\and
\textbf{Eric P. Xing} \\
Petuum, Inc. and \\
Carnegie Mellon University \\
\textsc{\href{mailto:epxing@cs.cmu.edu}{epxing@cs.cmu.edu}}
}

\maketitle

\begin{abstract}
Markov Chain Monte Carlo (MCMC) methods such as Gibbs sampling are finding widespread use in applied statistics and machine learning.
These often require significant computational power, and are increasingly being deployed on parallel and distributed systems such as compute clusters.
Recent work has proposed running iterative algorithms such as gradient descent and MCMC in parallel \emph{asynchronously} for increased performance, with good empirical results in certain problems.
Unfortunately, for MCMC this parallelization technique requires new convergence theory, as it has been explicitly demonstrated to lead to divergence on some examples.
Recent theory on \emph{Asynchronous Gibbs sampling} describes why these algorithms can fail, and provides a way to alter them to make them converge.
In this article, we describe how to apply this theory in a generic setting, to understand the asynchronous behavior of any MCMC algorithm, including those implemented using parameter servers, and those not based on Gibbs sampling. 
\\\strut
\end{abstract}

\section{Introduction}

Drawing samples from a posterior probability distribution for Bayesian learning is a fundamental task in today's statistical practice, machine learning, and data science.
Indeed, Bayesian methods including Gaussian Processes, Dirichlet Processes, mixed-effects regression, and many others, have found increased industrial application in recent years -- Latent Dirichlet Allocation \cite{blei03} is one widely-deployed example.

Unfortunately, Markov Chain Monte Carlo (MCMC) methods -- the cornerstone of modern Bayesian computation over the last two decades -- often do not scale well with data size or model complexity.
Recent work has sought to address this issue by deploying these methods on parallel computational hardware such as GPUs \cite{terenin16b} and compute clusters \cite{newman09, ihler12}.

In the systems community, there has been recent interest in \emph{asynchronous} approaches to machine learning, particularly for optimization-based methods.
In an asynchronous algorithm, worker nodes will perform computation \emph{as fast as they can}, without waiting on updates from other workers to arrive.
\textcite{ho13} and \textcite{wei15} have proposed a distributed system architecture called a \emph{parameter server} for running machine learning algorithms in this fashion.
Asynchronous methods can utilize hardware more effectively than fully synchronous methods, because worker nodes don't spend time waiting for other nodes, but require detailed attention because their convergence is in general not implied by standard optimization and Monte Carlo theory.

For Monte Carlo, asynchronous methods were first considered by the Latent Dirichlet Allocation (LDA) community, applied to Gibbs sampling.
In particular, \textcite{newman09} proposed an \emph{Asynchronous Gibbs sampler} for LDA \cite{terenin17b}, which was subsequently analyzed in an LDA-specific fashion by \textcite{ihler12}.
\textcite{johnson13} showed that Asynchronous Gibbs can diverge on certain Gaussian targets.
The first fully generic theory developed was by \textcite{terenin16a}, where it was shown that if workers are allowed to occasionally reject messages from other workers in a systematic way, then Asynchronous Gibbs Sampling converges.
Finally, \textcite{desa16}, whose work appeared concurrently, analyzed Asynchronous Gibbs in a completely different manner.

In this article, we outline how to apply the framework in \textcite{terenin16a} beyond the context of Gibbs sampling.
We prove that Asynchronous MCMC in a \emph{shared memory} setting converges without any additional assumptions.
We show how these ideas extend to the \emph{compute cluster} setting, using the framework to prove a new convergence result for Asynchronous MCMC on parameter servers.

This article is purely theoretical: there will be no data or discussion of performance. 
Our aim is to present and illustrate the types of analyses that are possible with current theory.

\section{Asynchronous Markov Chain Monte Carlo} \label{intro}

\begin{figure}[t]
\centerline{
\xymatrix@R-2ex{
\text{Fully Synchronous}:\hspace*{-6ex} & x_0 \ar[r]^P & x_1 \ar[r]^P & x_2 \ar[r]^P & x_3 \ar[r]^P & x_4 \ar[r]^P & x_5 \ar[r]^P & x_6 \ar[r]^P & x_7
\\
\text{Asynchronous}:\hspace*{-10ex} & x_0 \ar[r]_P \ar@/^/[rr]^P & x_1 \ar@/_/[rr]_P & x_2 & x_3 \ar[r]_P \ar@/^/[rrr]^P & x_4 \ar[r]_P & x_5 \ar@/_/[rr]_P & x_6 & x_7
}
}
\caption{Iterations of fully synchronous MCMC and asynchronous MCMC with shared memory.} \label{async-mcmc-shmem}
\end{figure}
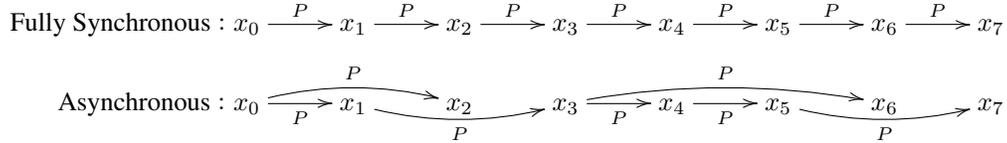

Before analyzing asynchronous MCMC, we first describe notation.
We begin by defining a convergent fully synchronous algorithm, on which subsequent asynchronous algorithms will be based.

\begin{definition} \label{underlying-chain}
Let $\Omega$ be a topological space representing the parameter space for the given problem.
Let $\mathscr{M}$ be the space of Borel probability measures over $\Omega$.
Let $\pi \in \mathscr{M}$ be the target probability measure, and let $\mu \in \mathscr{M}$ be an arbitrary measure.
Define a Markov operator $P: \mathscr{M} \goesto \mathscr{M}$ satisfying
\[
||P^k(\mu) - \pi)||_{\text{TV}} \goesto 0
\]
as $k \goesto \infty$, where $\text{TV}$ denotes total variation distance.
\end{definition}

To unpack this definition, consider an example.
If we take $\Omega = \R$ and $\pi$ to be Gaussian, then $P$ could be a Metropolis-Hastings algorithm targeting said Gaussian.
At time $k=0$, if $\mu$ is the distribution of the algorithm, then its distribution at time $k=1$ is $P(\mu)$.
It is important to highlight that $P$ converges on the space of measures $\mathscr{M}$, not on $\Omega$.
Figure \ref{async-mcmc-shmem} illustrates one possible trajectory of $P$ in a fully synchronous setting.

\subsection{Asynchronous MCMC on multiple threads with shared memory}

Now, we consider extending the above framework to the simplest asynchronous setting: a multithreaded computer that stores the current algorithm's state in shared memory.
The process will run according to the following procedure in parallel on every thread.

\begin{algorithm} \label{async-mcmc-shmem-algo}
For each thread do the following in parallel with no synchronization.
\begin{enumerate}
\item Read a value $x$ with distribution $\mu$ from shared memory.
\item Compute the random variable $x^*$ with distribution $\mu^* = P(\mu)$.
\item Write $x^*$ to shared memory, so that its distribution becomes $\mu^*$.
\end{enumerate}
\end{algorithm}

Algorithm \ref{async-mcmc-shmem-algo} is illustrated in Figure \ref{async-mcmc-shmem}.
We may understand it from two dual perspectives.
\begin{enumerate}
\renewcommand{\theenumi}{\alph{enumi}}
\renewcommand{\labelenumi}{(\theenumi)}
\item For a fixed value $x \in \Omega$, the next value $x^*$ is a random variable. \label{random-perspective}
\item For a fixed measure $\mu \in \mathscr{M}$, the next measure $\mu^*$ is also fixed. \label{deterministic-perspective}
\end{enumerate}

Hence, an MCMC algorithm is a random algorithm with respect to $\Omega$, but a non-random algorithm with respect to $\mathscr{M}$.
In the fully synchronous setting, perspective (\ref{random-perspective}) tends to be more powerful, because it allows us to analyze $P$ using stochastic process theory.
In the asynchronous setting, working with this perspective becomes very challenging because the random variables representing states of the algorithm no longer possess the Markov property.
We will thus utilize perspective (\ref{deterministic-perspective}) in this analysis.
The following property will be of interest.

\begin{result} \label{monotonic-convergence}
Suppose that $P$ is a Markov operator with unique limiting distribution $\pi$ and initial distribution $\mu$. Then
\[ \label{asynchronous-computation-2}
||P^{k+1} (\mu) - \pi||_{\f{TV}} \leq ||P^k (\mu) - \pi||_{\f{TV}}
\]
and we say that $P$ is a contracting operator with respect to the total variation metric.
\end{result}

\begin{proof}
\textcite{meyn93}, Proposition 13.3.2.
\end{proof}

For any Markov chain that converges to its stationary distribution, this implies that as $k \goesto \infty$, we have $||P^k(\mu) - \pi)||_{\text{TV}} \goesdownto 0$.
We can use this property to prove that Algorithm \ref{async-mcmc-shmem-algo} converges.

\newcommand{\TheoremFour}{
Let $P$ be a Markov operator with unique limiting distribution $\pi$.
Let $x_0 \dist \mu_0$ for some arbitrary $\mu_0$.
Let $x_1, x_2, ..$ be the sequence of values written to shared memory by the threads, and let $\mu_1, \mu_2, ..$ be their distributions.
Assume that there are no partial reads or writes of $x$.
Assume that $P$ remains time-homogeneous if executed asynchronously.
Assume that the maximum time between two updates made by the same thread is bounded above by $b$.
Then we have that $\underset{k \goesto \infty}{\lim} \mu_k = \pi$.
\vspace*{-1ex}
}
\begin{theorem} \label{async-mcmc-shmem-conv}
\TheoremFour
\end{theorem}

\begin{proof}
Appendix A. 
The key idea is that applying $P$ can only reduce the distance to stationarity, which causes the asynchronous algorithm to inherit convergence from its fully synchronous counterpart.
\end{proof}

Theorem \ref{async-mcmc-shmem-conv} shows us how convergence may be analyzed in settings where we do not have the Markov property.
The approach is appealing because every convergent MCMC algorithm's Markov operator is a contracting operator with respect to the total variation metric, and so it is completely general.

The main drawbacks of Theorem \ref{async-mcmc-shmem-conv} are that it requires threads to read and write the state $x$ in full, and that $x$ can only exist in one central location.
This limits the practical utility of Algorithm \ref{async-mcmc-shmem-algo}, because it makes it difficult to apply in a distributed setting.

\subsection{Asynchronous MCMC on a compute cluster}

\begin{figure}[t]
\centerline{
\xymatrix@R-2ex{
\text{Worker 1}:\hspace*{-6ex} & x_0 \ar[r] \ar[dr] & x_1 \ar[r] \ar[ddr] & x_2 \ar[r] & x_3 \ar[r] \ar[ddr] \ar[drr] & x_4 \ar[r] & x_5 \ar[r] \ar[drr] & x_6 \ar[r] & x_7
\\
\text{Worker 2}:\hspace*{-6ex} & y_0 \ar[r] & y_1 \ar[r] \ar[ur] & y_2 \ar[r] \ar[dr] & y_3 \ar[r] \ar[urr] & y_4 \ar[r] & y_5 \ar[r] \ar[dr] \ar[ur] & y_6 \ar[r] \ar[dr] & y_7
\\
\text{Worker 3}:\hspace*{-6ex} & z_0 \ar[r] \ar[urr] & z_1 \ar[r] \ar[uur]  & z_2 \ar[r] & z_3 \ar[r] \ar[ur] & z_4 \ar[r] & z_5 \ar[r] \ar[uur] & z_6 \ar[r] & z_7
}
}
\caption{Iterations of an asynchronous Gibbs sampler on a compute cluster.} \label{async-mcmc-cluster}
\end{figure}
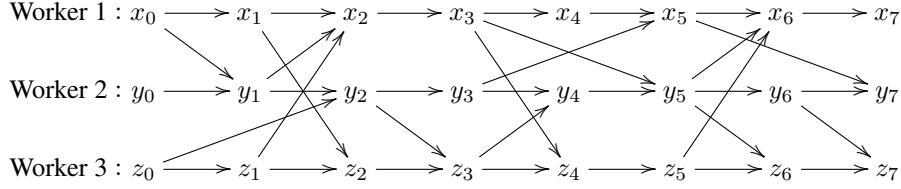

We now show how to extend the above ideas to the compute cluster setting, where the state $x$ may exist in multiple locations concurrently.
We begin by introducing the framework of \textcite{baudet78}, \textcite{bertsekas83}, and \textcite{frommer00}, as applied to MCMC by \textcite{terenin16a}.

\newcommand{\ResultFive}{
Suppose that $H: E \goesto E$ is a weakly coupled Markov operator with unique coupled limiting distribution $\Pi$ -- these refer to the entire cluster, whereas $\pi$ and $\mathscr{M}$ refer to each worker.
Let $m$ be the number of workers, assumed fixed and finite.
Assume the following.
\[
\text{Box Condition: we have } E = \bigtimes_{i=1}^m \mathscr{M}
.
\nonumber
\]
Let $H_i$ be the $i$th component of $H$ with respect to $E$ -- this is each worker's component of the coupled operator.
Assume $H$ is constructed from a Markov operator $P$ such that it inherits Result \ref{monotonic-convergence} with respect to each $H_i$.
Under appropriate regularity conditions given in \textcite{terenin16a}, $H$ will converge asynchronously in the sense of \textcite{frommer00}.
}
\begin{result} \label{async-conv}
\ResultFive
\end{result}

\begin{proof}
\textcite{terenin16a}.
\end{proof}

The main challenge in using this result is satisfying the box condition, particularly since an arbitrary measure space is not a product space.
One way to do this is to embed the target distribution of interest $\pi$ within some larger $\Pi$ that is contained in a product space, for instance by taking
\[
\Pi = (\pi,..,\pi).
\]
Having done this, we need to construct an $H$ that converges to $\Pi$.
We can do so using the Metropolis-Hastings method.
For details, see \textcite{terenin16a}, where such chains are proven to converge.
In that work, a convergent Asynchronous Gibbs sampler is constructed by taking $H_i$ to be a Metropolis-Hastings proposal based on a Gibbs step -- this is illustrated in Figure \ref{async-mcmc-cluster}.

There are many possible alternatives.
We can both choose $E$ to accommodate different kinds of parallel and distributed systems, and $H$ to accommodate other MCMC schemes.
For instance, consider the parameter server distributed system architecture of \textcite{ho13, wei15}. 
Here, we have a main node called the parameter server that allows worker nodes to read a state $x$, and send updated states $x^*$.
Upon receiving $x^*$, the parameter server updates its state using a pre-defined rule.
Consider now the following algorithm executed on a parameter server.

\begin{algorithm} \label{async-ps-algo}
For a set of workers, given an unnormalized target posterior distribution $\pi$, do the following in parallel with no synchronization.
\begin{enumerate}
\item Read $\del{x, \pi(x)}$ from the parameter server, and let $\mu$ be the distribution of $x$.
\item Compute the random variable $x^*$ using a valid MCMC step and evaluate $\pi(x^*)$.
\item Send $\del{x, x^*, \pi(x^*)}$ to the parameter server. The parameter server accepts the message and sets $x_s = x^*$ with Metropolis-Hastings probability 
\[
\min\cbr{1, \frac{\pi(x^*) f(x_s \given x)}{\pi(x_s) f(x^* \given x)}}
\]
where $x_s$ is the value of $x$ on the parameter server prior to receiving the message. 
\end{enumerate}
\end{algorithm}

As before, since the proposal distribution can depend on out-of-date states, Algorithm \ref{async-ps-algo} cannot be directly analyzed in a Markov framework.
Result \ref{async-conv} allows us to analyze convergence by pretending that communication is instantaneous, thereby allowing us to recover the Markov property and use it indirectly.
Using this idea, proving that Algorithm \ref{async-ps-algo} converges is straightforward: it essentially amounts to checking that the construction in \textcite{terenin16a} can be modified to accommodate the needed communication pattern, which is close to immediate.

It is often possible to avoid evaluating the data-dependent terms $\pi(x^*), \pi(x_s)$ on the parameter server. 
If we use, for example, a Hamiltonian Monte Carlo step, then $\pi(x)$ is immediately available from the worker node's computation.
This scheme directly addresses one of the disadvantages of Algorithm \ref{async-mcmc-shmem-algo} by allowing each worker to have a copy of the last state it saw. 
It also allows the algorithm to complete the burn-in phase more quickly by rejecting unhelpful updates from out-of-date workers.
Since our aim in this work is to showcase theory, we defer evaluation of performance to future work.

There are many other possible choices. 
Here, we have illustrated how analysis of Markov Chain Monte Carlo methods may be performed in the asynchronous distributed setting.
We hope that future work discovers more such methods.

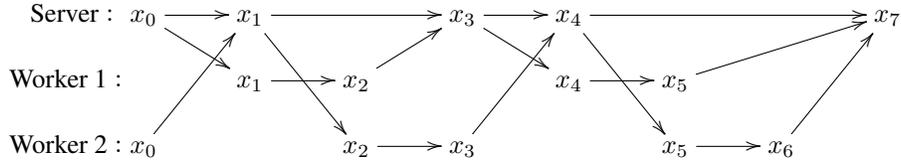
\begin{figure}[t]
\centerline{
\xymatrix@R-3ex{
\text{Server}:\hspace*{-7ex} & x_0 \ar[r] \ar[dr] & x_1 \ar[rr] \ar[ddr] & & x_3 \ar[r] \ar[dr] & x_4 \ar[rrr] \ar[ddr] & & & x_7
\\
\text{Worker 1}:\hspace*{-6ex} &  & x_1 \ar[r]  & x_2 \ar[ur] &  & x_4 \ar[r] & x_5  \ar[urr] &  & 
\\
\text{Worker 2}:\hspace*{-6ex} & x_0 \ar[uur] &   & x_2 \ar[r] & x_3 \ar[uur] &  & x_5 \ar[r] & x_6 \ar[uur] &
}
}
\caption{Iterations of an asynchronous MCMC sampler on a parameter server.} \label{async-mcmc-ps}
\end{figure}

\subsubsection*{Acknowledgments}

We are grateful to David Draper, Willie Neiswanger, and Hao Zhang for their thoughts.

\newpage
\printbibliography

\section*{Appendix A: proof of Theorem \ref{async-mcmc-shmem-conv}}

{
\renewcommand{\thetheorem}{\ref{async-mcmc-shmem-conv}}
\begin{theorem}
\TheoremFour
\end{theorem}
}

\begin{proof}
Fix an arbitrary $\mu_0$. 
Fix $b \in \N$.
We assume the maximum time between two updates made by the same worker is bounded in the sense that we have $\mu_{k+1} = P(\mu_j)$ for some $k-b < j \leq k$: call this the \emph{no worker dies} assumption.
Define the values $d_k = ||\mu_k - \pi||_{\text{TV}}$.
Now, consider
\[
d^*_k = \max\{d_j : k-b < j \leq k\}
\]
defined for $j > b$.
We claim that $d^*_k \geq d^*_{k+1}$.
First, note that
\[ 
d^*_{k+1} = \max\{d_j : k-b+1 < j \leq k+1 \} = \max\cbr{\max\{d_j : k-b+1 < j \leq k \}, d_{k+1}}
.
\label{max-max}
\]
Consider the left-hand term inside $\max$. 
We have that
\[
d^*_k = \max\{d_j : k-b < j \leq k \} \geq \max\{d_j : k-b+1 < j \leq k \}
\]
because $\max$ is taken over a smaller domain. 
Now, consider the right-hand term.
By the \emph{no worker dies} assumption and Result \ref{monotonic-convergence}, we have that
\[
d^*_k = \max\{d_j : k-b < j \leq k \} \geq d_{k+1}
.
\]
Together, this establishes that $d^*_k \geq d^*_{k+1}$.
Since by definition $d^*_k \geq 0$, the sequence $d^*_k$ is monotone and bounded, therefore it converges.
We now claim that there exists a subsequence
\[
d^*_{l_k} \text{ such that } d^*_{l_k} \goesto 0 \text{ as } k \goesto \infty.
\]
For each value $\mu_k$, let $p_k$ denote the number of times the Markov transition kernel $P$ was applied to obtain $\mu_k$, i.e. if $\mu_k = P^3(\mu_0)$ then $p_k=3$.
Let
\[
p^*_k = \min\{p_j : k-b < j \leq k\}
\]
which is nondecreasing and unbounded by the \emph{no worker dies} assumption -- to see this formally, decompose $\min$ as in (\ref{max-max}).
Define the sequence of times
\[
l_k = \arg\max\{p^*_j : 0 < j < k\} 
\]
which must also be unbounded.
But then we have $P^{l_k + 1}(\mu_0) = P[P^{l_k}(\mu_0)]$, so we can see by applying Definition \ref{underlying-chain} that as $k \goesto \infty$ we have
\[
||P^{l_k}(\mu_0) - \pi|| \goesto 0 \text{ and hence } d^*_{l_j} \goesto 0
\]
which gives the desired subsequence.
Finally, since $d^*_k$ is monotone, bounded below by 0, and admits an infinite subsequence that converges to 0, we have that $d^*_k \goesdownto 0$.
Since $d^*_k \geq d_k$, this implies $d_k \goesto 0$, which completes the proof.
\end{proof}

\end{document}